\documentclass[10pt,twocolumn,letterpaper]{article}

\usepackage{cvpr}              

\usepackage{multirow}
\usepackage{xcolor}
\usepackage{colortbl}
\usepackage{amsmath}
\usepackage{amsthm}
\usepackage{stfloats}

\theoremstyle{definition}

\newtheorem{theorem}{Theorem}

\definecolor{cvprblue}{rgb}{0.21,0.49,0.74}
\usepackage[pagebackref,breaklinks,colorlinks,allcolors=cvprblue]{hyperref}

\def\paperID{15310} 
\def\confName{CVPR}
\def\confYear{2026}

\title{MASQuant: Modality-Aware Smoothing Quantization for Multimodal Large Language Models}

\author{
Lulu Hu\thanks{Equal contribution.} \and
Wenhu Xiao\footnotemark[1] \and Xin Chen\footnotemark[1] \and Xinhua Xu \and
Bowen Xu\textsuperscript{\textdagger} \and Kun Li \and Yongliang Tao \\
Alibaba Cloud Computing, Alibaba Group\\
{\tt\small \{chudu.hll, wenhu.xwh, andy.cx, bowen.xbw\}@alibaba-inc.com}
}

\begin{document}
\maketitle

\begingroup
\renewcommand\thefootnote{\textdagger}
\footnotetext{Corresponding author}
\endgroup

\begin{abstract}
 Post-training quantization (PTQ) with computational invariance for Large Language Models~(LLMs) have demonstrated remarkable advances, however, their application to Multimodal Large Language Models~(MLLMs) presents substantial challenges. In this paper, we analyze SmoothQuant as a case study and identify two critical issues: Smoothing Misalignment and Cross-Modal Computational Invariance. To address these issues, we propose Modality-Aware Smoothing Quantization (MASQuant), a novel framework that introduces (1) Modality-Aware Smoothing (MAS), which learns separate, modality-specific smoothing factors to prevent Smoothing Misalignment, and (2) Cross-Modal Compensation (CMC), which addresses Cross-modal Computational Invariance by using SVD whitening to transform multi-modal activation differences into low-rank forms, enabling unified quantization across modalities. MASQuant demonstrates stable quantization performance across both dual-modal and tri-modal MLLMs. Experimental results show that MASQuant is competitive among the state-of-the-art PTQ algorithms. Source code: \url{https://github.com/alibaba/EfficientAI}. 

\end{abstract}

\section{Introduction}
Post-training quantization~(PTQ) has become essential for deploying Large Language Models~(LLMs)~\citep{Qwen2, llama3, gpt4} on resource-constrained devices, and this need is amplified for Multimodal Large Language Models~(MLLMs), which demonstrate impressive cross-modal reasoning capabilities~\citep{kahn2020libri, wenetspeech, singh2019towards, fu2025video, mmmu, omnibench}. PTQ methods based on computational invariance~\citep{awq, smoothquant, omniquant, os, affinequant, flatquant}, particularly channel-wise smoothing~\citep{smoothquant, awq, omniquant}, have proven highly effective for text-only LLMs by redistributing activation outliers through channel-level scaling factors. Recent work has also begun recognizing MLLM-specific characteristics—MBQ~\citep{mbq} observes unequal contributions of visual and text tokens to quantization error, while MQuant~\citep{mquant} finds that visual token activations exhibit much higher magnitudes than text tokens. However, the direct application of channel-wise smoothing to MLLMs remains surprisingly underexplored, raising an important question: do these successful channel-wise smoothing PTQ methods transfer seamlessly to the multimodal setting?

Through systematic analysis of vision-language and omni-modal MLLMs~\citep{qwen2.5-vl,qwen2.5omni}, we identify a fundamental problem. Different modalities exhibit vastly different activation magnitudes—visual tokens typically show ranges 10–100× larger than text and audio tokens. Channel-wise smoothing computes a single scaling factor per channel, but when modalities with such disparate distributions pass through the same layer, the dominant modality's larger activations dictate the smoothing factor. Activations from non-dominant modalities thus become over-smoothed, crushing their signal and causing severe quantization errors. We term this phenomenon smoothing misalignment. A natural solution is to compute separate smoothing factors for each modality. However, this seemingly simple solution introduces a critical flaw: preserving computational invariance under this scheme requires storing distinct quantized weights for each modality. This defeats the fundamental purpose of quantization, which aims to reduce memory footprint through a single low-precision weight representation. The question becomes: \textit{can we collect modality-specific smoothing factors while maintaining a single quantized weight for inference}?

We address this challenge through Modality-Aware Smoothing Quantization ~(MASQuant). Our key idea is to learn dedicated smoothing factors for each modality, but during inference, we use the text-smoothed weights as a base and apply modality-specific low-rank compensation. This design simultaneously resolves smoothing misalignment and preserves computational invariance. Specifically, Modality-Aware Smoothing (MAS) optimizes smoothing factors directly for each modality's activations, eliminating smoothing misalignment and pushing channel-wise smoothing to its optimization limit. To maintain a single weight representation, Cross-Modal Compensation (CMC) leverages a key observation: differences in smoothed activations across modalities are low-rank. We prove this mathematically and use SVD-based whitening to transform these differences into compact low-rank matrices. By compensating for modality-induced variations through lightweight low-rank corrections to the text-smoothed base weights, MASQuant achieves modality-specific adaptation without sacrificing the unified weight structure essential for efficient quantization.

We evaluate MASQuant across diverse MLLM architectures spanning vision-language and omni-modal configurations. Results across all evaluated benchmarks demonstrate consistent superiority over existing channel-wise smoothing PTQ methods. 

In summary, our main contributions are:
\begin{itemize}
\item We identify and formalize smoothing misalignment—the fundamental obstacle in applying channel-wise smoothing PTQ to MLLMs—and resolve it through Modality-Aware Smoothing.

\item We prove that inter-modal activation differences are low-rank, enabling Cross-Modal Compensation to maintain computational invariance with a single set of quantized weights.

\item We present MASQuant, a PTQ method that is effective on both vision-language and omni-modal MLLMs.
\end{itemize}

\section{Related Work}
\paragraph{LLMs Quantization.} LLM quantization methods are broadly categorized into \emph{Quantization-Aware Training} (QAT)~\citep{LLM-QAT,Efficientqat,lora-qat} and \emph{Post-Training Quantization} (PTQ)~\citep{awq,smoothquant,omniquant}. QAT incorporates quantization into training to adapt models to low-precision computation, while PTQ applies quantization directly using calibration data. PTQ approaches include: (1) error compensation via second-order gradients~\citep{gptq} or low-rank correction~\citep{zeroquantv2,aser}; (2) channel-wise smoothing to mitigate outliers~\citep{smoothquant,awq,omniquant,affinequant,flatquant}; (3) rotation-based distribution restructuring~\citep{quip,quarot,spinquant}; and (4) mixed-precision strategies~\citep{mixllm,spqr,squeezellm}.

\paragraph{MLLMs Quantization.}
Quantizing MLLMs presents unique challenges due to cross-modal activation disparities. MQuant~\citep{mquant} identifies that visual token activations can exceed textual ones by 20×, proposing modality-specific quantization. MBQ~\citep{mbq} observes visual tokens are less quantization-sensitive and introduces gradient-weighted error balancing. QSLAW~\citep{qslaw} addresses increased outlier density from multimodal inputs through learnable weight-group scaling. Despite these advances, activation quantization remains inadequately addressed, motivating our modality-aware smooth quantization approach.

\section{Preliminaries}
\label{sec:Preliminary}

\subsection{Computational Invariance based PTQ}
PTQ addresses significant computational and storage challenges by mapping high-precision floating-point tensor $\mathbf{x}$ to low-precision $\mathrm{N}$-bit integer tensor $\hat{\mathrm{x}}_\mathrm{N}$:
\begin{equation}
\label{eq:PTQ}
\hat{\mathrm{x}}_\mathrm{N} = \mathrm{Q}(\mathrm{x}) = \left(\mathrm{clamp}\left(\left\lfloor\frac{\mathrm{x}}{\Delta}\right\rceil + \mathrm{z}, \mathrm{q_{min}}, \mathrm{q_{max}}\right) - \mathrm{z}\right) \cdot \Delta,
\end{equation}
where $\Delta$ is the scale factor, $\mathrm{z}$ is zero-point, $\lfloor\cdot\rceil$ is the rounding-to-nearset operator, and $\text{clamp}$ clips values outside the integer range $[\mathrm{q_{min},q_{max}}]$. We use W$x$A$y$ notation for $x$-bit weights and $y$-bit activations, with two main types: weight-only quantization (e.g., W4A16) and weight-activation quantization (e.g., W8A8, W4A8). For a linear layer $\mathbf{Y} = \mathbf{X}\mathbf{W}$, where $\mathbf{X} \in \mathbb{R}^{T\times D_\text{in}}$,  $\mathbf{W} \in \mathbb{R}^{D_\text{in} \times D_{\text{out}}}$, the layer can be reformulated based on computational invariance as:
\begin{equation}
    \mathbf{Y} = (\mathbf{X} \mathbf{S}^{-1}) \cdot (\mathbf{S} \mathbf{W}),
\end{equation}
where $\mathbf{S}$ can be diagonal matrix \citep{smoothquant, awq, omniquant} or orthogonal matrix \citep{spinquant, flatquant, duquant, quip}. $\mathbf{S}$ reduces the outliers in $\mathbf{X}$ and lead to better quantization reconstruction loss:
\begin{equation}
    L= \mathcal{L}(\mathrm{Q}(\mathbf{X} \mathbf{S}^{-1}) \cdot \mathrm{Q}(\mathbf{S} \mathbf{W}), \mathbf{X}\mathbf{W}).
\end{equation}
 In this paper, we demonstrate that applying $\mathbf{S}$ for different modalities in MLLMs can achieve robust and effective PTQ performance.

\subsection{SVD-based Whitening}
Prior work~\citep{svdllm, svdllm-v2} uses SVD-based whitening for low-rank weight compression. Given input activations $\mathbf{X}$ and weights $\mathbf{W}$, the whitening transform is derived by decomposing the activation covariance:
\begin{equation}
    \mathbf{P\Lambda P}^\top = \mathrm{SVD}(\mathbf{X}^\top\mathbf{X}).
\end{equation}
A whitening matrix $\mathbf{T} = (\mathbf{P\Lambda}^{1/2})^\top$ transforms $\mathbf{XW}$ into $\mathbf{(XT)(T^{-1}W)}$ which satisfies $\mathbf{(XT^{-1})^\top(XT^{-1}) = I}$, yielding whitened activations. SVD-LLM v2~\citep{svdllm-v2} shows that performing SVD on $\mathbf{TW}$ and truncating to rank $r$ minimizes reconstruction error $\|\mathbf{XW - XW'}\|_F$:
\begin{gather}
\mathbf{U, S, V} = \mathrm{SVD}(\mathbf{TW}), \\
\mathbf{U_r, S_r, V_r} = \mathrm{Trunc}_r(\mathbf{U, S, V}), \\
\mathbf{W'} = \mathbf{T^{-1}U_rS_rV_r}.
\end{gather}
While existing methods apply whitening for weight compression, we are the first to show that whitening can effectively compensate for cross-modal weight differences, enabling unified quantization across modalities.

\begin{figure*}[ht!]
    \centering
    \includegraphics[width=1.0\textwidth]{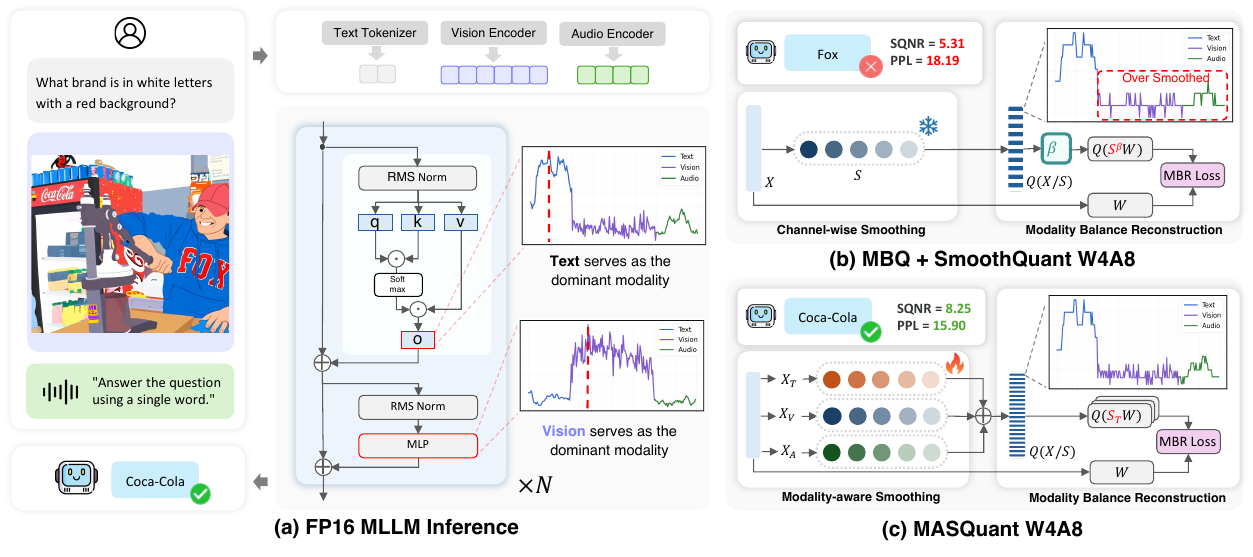}
    \caption{ 
    (a). Activation distributions during multimodal reasoning in MLLMs. Different dominant modalities emerge across MLLM components, leading to failure of general PTQ methods that diminish vision importance. (b). Impact of SmoothQuant's uniform smoothing factors $S$ computation on MLLM quantization performance (low SQNR, high PPL). (c). MASQuant addresses smoothing misalignment through the combination of MAS and CMC, thereby significantly enhancing PTQ performance in MLLMs. MBR Loss indicates Modality Balanced Reconstruction Loss.}
    \label{fig:Smoothing}
\end{figure*}
\section{MASQuant}
\label{sec:contrastive_experiments}

In this section, we present our Modality-Aware Smooth Quantization~(MASQuant) framework includes Modality-Aware Smoothing~(MAS) followed by Cross-Modal Compensation~(CMC), designed to address the critical issue of smoothing misalignment and cross-modal computational invariance in MLLMs.  

\subsection{Motivation}
\label{subsec:misalignment}
\paragraph{Revisiting Smoothing Factors.}
Existing methods compute smoothing factors $s$ through closed-form solutions. SmoothQuant~\citep{smoothquant} uses:
\begin{equation}
\label{eq:sq_sfs}
s_i = \frac{\max_t |x_{t,i}|^\beta}{\max_j |w_{j,i}|^{1-\beta}},
\end{equation}
while AWQ~\citep{awq} adopts:
\begin{equation}
\label{eq:awq_sfs}
s_i = (\frac{\sum_{t=1}^n|x_{t,i}|}{n})^\beta, \beta^*=\textrm{arg}\min_\beta L(\mathrm{s}).
\end{equation}
MBQ~\citep{mbq} extends this to MLLMs by adjusting the contribution of different modality activations when searching for $\beta$ values. However, all these methods optimize only the hyperparameter $\beta$ while the smoothing factors themselves are never treated as free parameters. OmniQuant~\citep{omniquant} demonstrates that learnable smoothing factors can better minimize quantization error, we adopt this principle for MLLMs---instead of searching over $\beta$, we directly optimize matrix $\mathbf{S}$ whose diagonal entries are smoothing factors $s_i$:
\begin{equation}
    \mathbf{S}^* = \mathop{\arg\min}\limits_{\mathbf{S}} \mathcal{L}( \mathrm{Q}\left(\mathbf{X} \mathbf{S}^{-1} \right) \mathrm{Q}(\mathbf{S} \mathbf{W}),\mathbf{X} \mathbf{W}) .
\end{equation}
This formulation is more flexible—it can discover optimal smoothing patterns beyond what any $\beta$-parameterized formula can express.
\paragraph{Smoothing Misalignment.}
MLLMs process multiple modalities with vastly different activation magnitudes ~\citep{mquant}. we measure the activation range per channel: $\mathbf{R}^m_i = \max_{t}|\mathbf{x}_{t,i}^m|$ for modality $m \in \mathcal{M}$, where $\mathbf{x}_{t,i}^m$ is the activation of the $t$-th token in channel $i$.  When calibrating on mixed-modality data, the unified smoothing factor is determined by the maximum range across all tokens:
\begin{equation}
\mathbf{s}^{\text{uni}}_i = \frac{\left(\max_t |\mathrm{x}_{t,i}|\right)^\beta}{\left(\max_j |w_{j,i}|\right)^{1-\beta}} = \frac{\left(\max_{m,t} |\mathrm{x^{m}}_{t,i}|\right)^\beta}{\left(\max_j |w_{j,i}|\right)^{1-\beta}} \propto \mathbf{R}_i^{m^*},
\end{equation}
where $m^*=\arg\max_m\mathbf{R}_i^m$ is the dominant modality. This unified factor $\mathbf{s}^{\text{uni}}$ aligns with the dominant modality's ideal factor but severely mismatches others—a phenomenon we term smoothing misalignment. Non-dominant modalities suffer degraded quantization quality as their activations are scaled by factors optimized for a different magnitude regime. Figure~\ref{fig:Smoothing} illustrates this phenomenon in detail.

\begin{figure}[t!]
\centering
\begin{subfigure}{1.0\columnwidth}
\centering
\includegraphics[width=0.9\textwidth]{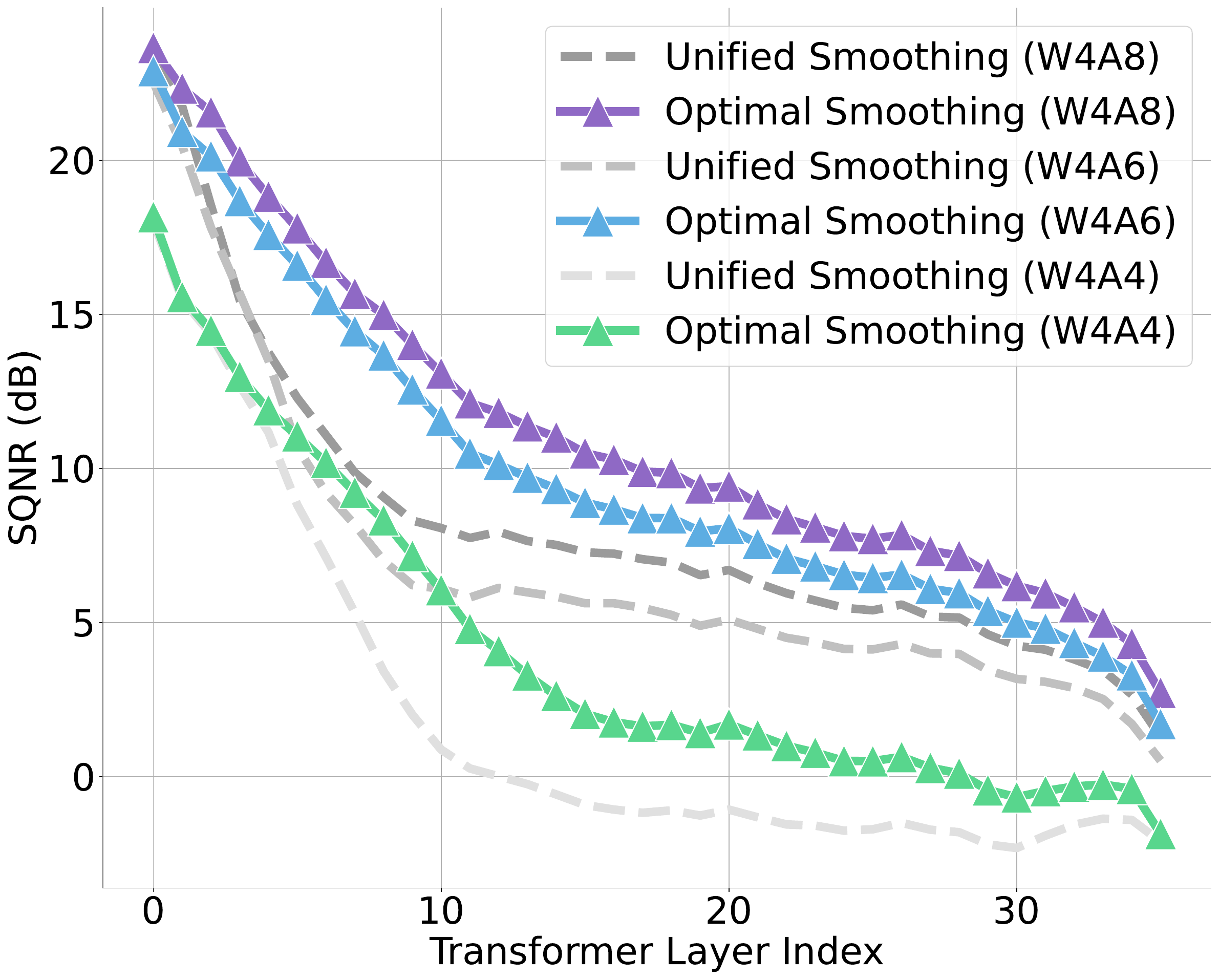}
\vspace{-3mm}
\end{subfigure}
\caption{Comparative analysis of SQNR degradation of Qwen2.5-Omni-3B under multimodal input condition. We selected 32 samples from OmniBench and computed the average SQNR for each layer.}
\label{fig:mas_sqnr}
\end{figure}

\subsection{Modality-Aware Smoothing}
Smoothing misalignment arises from unified scaling across heterogeneous modalities. We address this by maintaining modality-specific smoothing factors, eliminating dominance at its root. Let $\mathcal{M}$ denote the set of supported modalities (e.g., text, image, audio). For each modality $m \in \mathcal{M}$, we first obtain the initial values of the modality-aware smoothing factors  $\mathbf{S}_{m} \in \mathbb{R}^{d \times d}$:
\begin{equation}
\mathbf{S}_m = \text{diag}(\mathbf{s}^m), \quad 
s^m_i = \sqrt{\frac{\max_{t} |x^{m}_{t,i}|}{\max_j |w_{j,i}|}}, 
\quad m \in \mathcal{M} .
\end{equation}
We then optimize the $\mathbf{S}_{m}$ by minimizing MAE loss~\citep{mbq} on modality-specific data, We simplify the notation $\{\mathbf{S}_{m}\}_{m \in \mathcal{M}}$ as $\{\mathbf{S}_{m}\}$ for brevity:
\begin{equation}
    \{\mathbf{S}_{m}^*\} = \mathop{\arg\min}\limits_{\{\mathbf{S}_{m}\}} \sum_{m \in \mathcal{M}} (\lambda_m\cdot\mathcal{L}_{\mathrm{MAE}}(\mathbf{S}_m,\mathbf{X}_{m}, \mathbf{W})) , 
\end{equation}
where $\lambda_m$ denotes the loss weight for modality $m$. For modality $m$, the quantization reconstruct MAE loss is:
\begin{equation}
\label{eq:mas_quant}
  \mathcal{L}_\mathrm{MAE}= ||\mathrm{Q}\left(\mathbf{X}_{m} \mathbf{S}_{m}^{-1}\right) \cdot \mathrm{Q}\left(\mathbf{S}_{m} \mathbf{W}\right)-\mathbf{X}_{m} \mathbf{W}|| .
\end{equation}
 This ensures $\mathbf{S}_{m}^*$ captures modality-specific statistics without cross-modality interference. We quantify benefit of $\mathbf{S}_m$ through SQNR~\citep{sqnr1,sqnr2} analysis. For a token $\mathbf{x}_t$, SQNR measures quantization quality as:
\begin{equation}
    \mathrm{SQNR}(\mathbf{x_t}) = 10 \log_{10} \left( \frac{\|\mathbf{x_t}\|^2}{\|\mathbf{x_t} - \mathrm{Q}(\mathbf{x_t})\|^2} \right),
\label{eq:sqnr}
\end{equation}
Under channel-wise smoothing with smoothing factor $s$, when $\mathbf{x_t} \gg \Delta_t$ in Equation~\ref{eq:PTQ}, quantization error $e_t$ can be approximated as uniformly distributed over $\left[-\frac{\Delta_t}{2}, \frac{\Delta_t}{2}\right]$. Thus, the mean squared error across $d$ channels is $d \cdot \frac{\Delta_t^2}{12}$. This yields SQNR of SmoothQuant (omitting constants):
\begin{equation}
    \mathrm{SQNR}(\mathbf{s},\mathbf{x_t}) \propto \frac{\sum_{i=1}^{d} \left\| \frac{{x_{t,i}}}{{s_i}} \right\|^2}{d \cdot \frac{\Delta_t^2}{12}} \propto \frac{\sum_{i=1}^{d} \left( \frac{{x_{t,i}}}{{s_i}} \right)^2}{\left( \max_i \left| \frac{{x_{t,i}}}{{s_i}} \right| \right)^2},
\end{equation}
based on which the degradation can be quantified.
\begin{theorem}\label{SQNR theorem}[SQNR Degradation under Smoothing Misalignment]
\label{thm:general_degradation}
Consider a layer processing multimodal inputs with dominant modality $m$ and non-dominant modality $m'$. Let $\boldsymbol{\alpha}^{m,m'}_i = R^m_i / R^{m'}_i$ denote the range ratio at channel $i$, where $R^m_i$ and $R^{m'}_i$ are the activation ranges. Using unified smoothing $\mathbf{s}^{\text{uni}} \approx \mathbf{s}^m$ yields SQNR degradation (exact as $\alpha^{m,m'} \gg 1$):
\begin{equation}
\begin{aligned}
\mathrm{SQNR}(\mathbf{s}^{\text{uni}}, \mathbf{x}_t^{m'}) & = \mathrm{SQNR}(\mathbf{s}^{m'}, \mathbf{x}_t^{m'}) \\ &- 10\log_{10}\left(\frac{{d}(\min_i{(\boldsymbol\alpha^{m,m'}_i)^2})}{\sum_{i=1}^{d} \frac{1}{(\boldsymbol\alpha^{m,m'}_i)^2}}\right).
\label{eq:general_bound}
\end{aligned}
\end{equation}

\end{theorem}
\begin{proof}
To simplify notation, we let $x_i=x_{t,i}^{m'}, R=R_i^{m'},$and $\alpha_i=\alpha_i^{m,m'}$. We compare the upper bounds of the SQNR for the token $\mathbf{x}_t^{m'}$ under two smoothing strategies, which are achieved when the token's activations are uniform across channels(i.e., $x_{1}/R_{1} \approx x_{2}/R_{2} \approx \cdots \approx x_{d}/R_{d}$):
The SQNR for optimal smoothing is proportional to the dimension $d$:
\begin{equation}
\mathrm{SQNR}(\mathbf{s}^{m'}, \mathbf{x}) = 10\log_{10} \frac{\sum_{i=1}^{d} (x_i/R_i)^2}{(\max_i |x_i/R_i|)^2} = 10\log_{10}d, 
\end{equation}
and SQNR for unified smoothing upper bound is given by:
\begin{equation}
\begin{aligned}
\mathrm{SQNR}(\mathbf{s}^{\text{uni}}, \mathbf{x}) =\log_{10}\frac{\sum_{i=1}^{d} 1/\alpha_i^2}{(\max_i 1/\alpha_i)^2}.
\end{aligned}
\end{equation}
The difference between the two SQNR upper bounds is:
\begin{equation}
\begin{aligned}
\Delta = \underbrace{10\log_{10}\left(\frac{1}{d \cdot (\max_i \frac{1}{\alpha_i})^2} \sum_{i=1}^{d} \frac{1}{\alpha_i^2}\right)}_{\leq 0, \text{ equality iff } \alpha_1=\cdots=\alpha_d}. 
\end{aligned}
\end{equation}
\end{proof}
The results in Figure~\ref{fig:mas_sqnr} validate Theorem 1 and Figure~\ref{fig:range_ratio_omni} and~\ref{fig:range_ratio_vl} provide distribution of $\alpha_{i}^{m,m'}$ is non-uniform. Based on this, the inference process of the MAS becomes as follows:
 \begin{equation}
\mathbf{Y} = \mathrm{Q}\left(\mathbf{X}_{m} \mathbf{S}_{m}^{-1}\right) \cdot \mathrm{Q}\left(\mathbf{S}_{m} \mathbf{W}\right),\quad m \in \mathcal{M}.
\end{equation}



\begin{figure}[t!]
    \centering
    \includegraphics[width=0.99\linewidth]{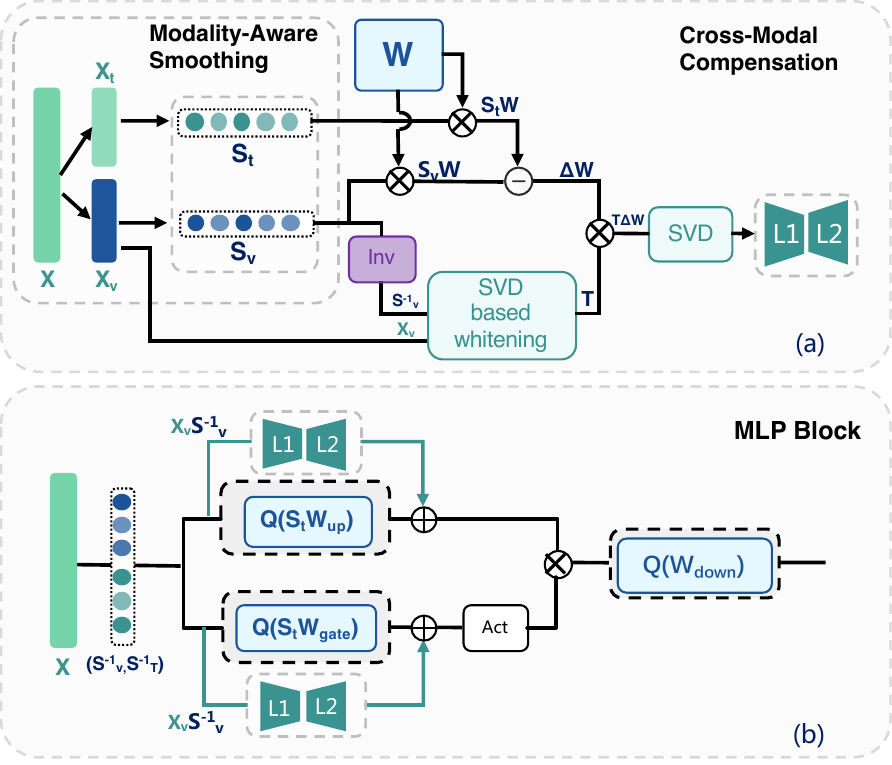}
    \caption{ 
    The illustrated case demonstrates a text-vision dual-modal setting. (a) Schematic workflow of MAS and CMC with calibration data, (b) Illustration of how low-rank matrices L1 and L2 in CMC are utilized in MASQuant, exemplified with an MLP block.}
    \label{fig:MASQuant}
\end{figure}

\subsection{Cross-Modal Compensation}
While MAS eliminates smoothing misalignment, it produces modality-specific quantized weight matrices $\mathrm{Q}(\mathbf{S}_m \mathbf{W})$ that violate computational invariance---PTQ requires a single quantized weight across all modalities.

Our strategy is to store only one quantized weight $\mathrm{Q}(\mathbf{S}_t\mathbf{W})$ using text smoothing as reference, and compensate for other modalities via low-rank corrections. Consider vision inputs: ideally we compute $\mathbf{X}_v\mathbf{S}_v^{-1} \cdot (\mathbf{S}_v\mathbf{W})$, but using the shared weight produces a residual:
\begin{equation}
    \mathbf{\Delta Y} = \mathbf{X}_v\mathbf{S}_v^{-1} \cdot \underbrace{(\mathbf{S}_v\mathbf{W} - \mathrm{Q}(\mathbf{S}_t\mathbf{W}))}_{\mathbf{\Delta W}}.
\end{equation}
A natural approach is low-rank approximation of $\mathbf{\Delta W}$. However, directly applying SVD fails: it does not minimize the output residual $\mathbf{\Delta Y}$, and $\mathbf{\Delta W}$ lacks low-rank structure.
Our key insight is that whitening the activations induces low-rank structure. We compute the whitening transform via:
\begin{equation}
\mathrm{SVD}\big((\mathbf{X}_v\mathbf{S}_v^{-1})^\top(\mathbf{X}_v\mathbf{S}_v^{-1})\big) = \mathbf{P\Lambda P}^\top, \quad \mathbf{T} = (\mathbf{P\Lambda}^{1/2})^\top.
\end{equation}
which ensures $(\mathbf{X}_v\mathbf{S}_v^{-1})\mathbf{T}^{-1}$ is orthonormal. The whitened residual $\mathbf{T}(\mathbf{\Delta W})$ exhibits strong low-rank structure (see Figure~\ref{fig:effective_rank}), enabling accurate approximation via truncated SVD:
\begin{equation}
    \mathrm{SVD}(\mathbf{T}(\mathbf{\Delta W})) = \mathbf{U\Sigma V}^\top \approx \mathbf{U}_r\mathbf{\Sigma}_r\mathbf{V}_r^\top.
\end{equation}
Inverting the whitening yields the low-rank correction:
\begin{equation}
\label{eq:l1l2}
    \mathbf{\Delta W} \approx \mathbf{L}_1\mathbf{L}_2, \quad \text{where} \quad \mathbf{L}_1 = \mathbf{T}^{-1}\mathbf{U}_r, \; \mathbf{L}_2 = \mathbf{\Sigma}_r\mathbf{V}_r^\top.
\end{equation}

\begin{theorem}[Optimal Low-rank Compensation]\label{CMC theorem}
The rank-$r$ matrix $\mathbf{L}^*=\mathbf{L_1}\mathbf{L_2}$, where $\mathbf{L_1}$, $\mathbf{L_2}$ are given by the rank-$r$ truncated SVD defined in Eq \ref{eq:l1l2}, minimizes the reconstruction loss $\mathcal{L}(\mathbf{L}) = \left\| \mathbf{X}_v\mathbf{S}_v^{-1} (\mathbf{\Delta W} - \mathbf{L}) \right\|_F^2$. i.e,
\begin{equation}
    \mathbf{L}^* = \operatorname*{arg\,min}_{\text{rank}(\mathbf{L}) \leq r}\left\| \mathbf{X}_v\mathbf{S}_v^{-1} (\mathbf{\Delta W} - \mathbf{L}) \right\|_F^2.
\end{equation}
\end{theorem}

\begin{proof}
Considering only two modalities, text and vision, and performing weight-only quantization exclusively, As defined, we have
\begin{equation}
    \mathbf{L}^* = \mathbf{T}^{-1} \left( \textrm{Trunc}_r(\mathbf{T} \mathbf{\Delta W}) \right),
\end{equation}
where $\text{Trunc}_r(\cdot)$ denotes the best rank-$r$ approximation obtained via truncated SVD., and the whitening matrix $\mathbf{T} = (\mathbf{P} \mathbf{\Lambda}^{1/2})^\top$.
On the other hand, the SVD decomposition of the covariance matrix is $(\mathbf{X}_v\mathbf{S}_v^{-1})^\top(\mathbf{X}_v\mathbf{S}_v^{-1})=\mathbf{P}\Lambda\mathbf{P}^\top$. It leads to $\mathbf{X}_v\mathbf{S}_v^{-1}=\mathbf{U}\mathbf{\Lambda}^{1/2} \mathbf{P}^\top=\mathbf{U}\mathbf{T}$, where $\mathbf{U}$ is orthonormal. Therefore, we have

\begin{equation}
\begin{aligned}
\mathcal{L}(\mathbf{L^*}) &= ||\mathbf{X}_v\mathbf{S}_v^{-1}(\mathbf{\Delta W}-\mathbf{L^*})||^2_F \\
&= ||\mathbf{UT}(\mathbf{\Delta W}-\mathbf{L^*})||^2_F \\
&= ||\mathbf{UT}(\mathbf{\Delta W}-\mathbf{T}^{-1}\textrm{Trunc}_r(\mathbf{T\Delta W}))||^2_F \\
&= ||\mathbf{T}\mathbf{\Delta W}-\textrm{Trunc}_r(\mathbf{T\Delta W})||^2_F \\
&= ||\text{SVD}(\mathbf{T}\Delta \mathbf{W})||_F^2 \\
&= ||\text{SVD}(\mathbf{U}^{-1}\mathbf{X}_v\mathbf{S}_v^{-1}\Delta \mathbf{W})||_F^2\\
&= ||\text{SVD}(\mathbf{X}_v\mathbf{S}_v^{-1}\Delta \mathbf{W})||_F^2 = L_{min}^2.
\end{aligned}
\end{equation}
So the designed SVD truncation ensures the theoretical minimum truncation loss.
\end{proof}
The theorem \ref{CMC theorem} shows that SVD-based whitening enables low-rank compensation to effectively bridge modality gaps. The final inference combines the base quantized output with modality-specific corrections:
\begin{equation}
\mathbf{Y} =
\begin{cases}
\mathrm{Q}(\mathrm{X}_m\mathbf{S}_m^{-1}) \cdot \mathrm{Q}(\mathbf{S}_t\mathbf{W}), & m = \text{text} \\
\begin{aligned}
&\mathrm{Q}(\mathbf{X}_m\mathbf{S}_m^{-1}) \cdot \mathrm{Q}(\mathbf{S}_t\mathbf{W}) \\
&\quad + \mathbf{X}_m\mathbf{S}_m^{-1} \cdot \mathbf{L}_1^m\mathbf{L}_2^m. \end{aligned} & m \neq \text{text} 
\end{cases}
\end{equation}
This maintains a single quantized weight for efficiency while using lightweight low-rank matrices to preserve accuracy across modalities. The complete pipeline is illustrated in Figure~\ref{fig:MASQuant}.

\begin{table*}[ht!]
\setlength{\tabcolsep}{4pt}
\centering
\caption{\small Comparison of MASQuant with existing quantization methods on multimodal benchmarks. OCR: OCRBench. Viz: Vizwiz. S-QA: ScienceQA. T-VQA: TextVQA. SQ: SmoothQuant. The best results are highlighted in \textbf{bold}.}
\label{tab:main_vl}
\resizebox{\textwidth}{!}{%
\begin{tabular}{@{}l|c|!{\vrule width 1pt}cccccc|!{\vrule width 1pt}cccccc@{}}
\noalign{\vspace{0.1em}}\toprule
  \textbf{Method} &  &\multicolumn{6}{c}{\textbf{Qwen2.5-VL-3B}} & \multicolumn{6}{c}{\textbf{Qwen2.5-VL-7B}}  \\
\hline
   &Bits & MMMU  & OCR & Viz & S-QA & T-VQA & Avg& MMMU & OCR &Viz & S-QA & T-VQA & Avg\\
  & &Acc$\uparrow$ & Acc$\uparrow$ & Acc$\uparrow$ & Acc$\uparrow$ & Acc$\uparrow$ & Acc$\uparrow$ &Acc$\uparrow$ & Acc$\uparrow$ & Acc$\uparrow$ & Acc$\uparrow$ & Acc$\uparrow$ & Acc$\uparrow$ \\
\noalign{\vspace{0.1em}}\hline\noalign{\vspace{0.1em}}
   $\text{Dense}$ & FP16 & 42.2  & 79.3 & 69.1 & 81.9 & 77.9 & 70.1 & 46.7  & 83.8 & 70.8 & 88.4 & 82.9 & 74.5\\
\hline
\noalign{\vspace{0.1em}}
 $\text{RTN}$ & \multirow{4}{*}{W4A16}& 40.9  & 77.9  & 63.0 & 81.3 & 75.8 & 67.8  &  43.3  & 83.7 & 67.8 & 81.3 & 82.1 & 71.6\\
  $\text{AWQ}$ &  & 41.9 & 78.6 & \textbf{68.1} & 81.7 & 77.2 & 69.5 & 43.3 & 83.7 & 70.6 & 87.8 & 82.2 & 73.5 \\
  $\text{MBQ}$ & & 41.9  & 78.5 & 68.0 & 81.5 & 76.8 & 69.3 & 44.4 & 82.8 & 70.6 & 87.8 & \textbf{82.9} & 73.7\\
  \rowcolor[HTML]{EAE8FD}$\text{MASQuant(ours)}$ & & \textbf{43.3}  & \textbf{78.6} & 67.7 & \textbf{82.4} & \textbf{77.3} & \textbf{69.9} & \textbf{44.4} & \textbf{84.6} & \textbf{71.5} & \textbf{87.8} & 82.5 & \textbf{74.2}\\
\hline
  $\text{RTN}$ & \multirow{4}{*}{W8A8} &42.6  & 78.7 & 68.2 & 82.4 & 77.3 & 69.8 & 45.6 & 83.8 & 70.5 & 88.1 & 82.5 & 74.1\\
  $\text{SQ}$ & & 42.6  & 79.0 & 68.0 & 82.2 & 77.5 & 69.9 & 43.3 & 83.8 & 70.0 & 88.2 & 82.6 & 73.6\\
  $\text{MBQ}$ & & 43.3  & 78.9 & \textbf{68.6} & 81.8 & 77.8 & 70.1 &  \textbf{46.7} & 83.5 & 70.6 & 88.5 & \textbf{82.9} & 74.4\\
  \rowcolor[HTML]{EAE8FD}$\text{MASQuant(ours)}$ & & \textbf{46.6} &\textbf{79.5} & 68.2 & \textbf{82.4} & \textbf{77.7} & \textbf{70.9} & 46.2 & \textbf{84.2} & \textbf{70.6} & \textbf{88.6} & 82.6 & \textbf{74.4}\\
\hline
 $\text{RTN}$ & \multirow{4}{*}{W4A8} & 25.6 & 0.0 & 0.0 & 0.0 & 0.0 & 5.1 & 43.3 &  68.3 & 63.2 & 85.2 & 76.9  & 67.4\\
  $\text{SQ}$ & & 25.6  & 66.9 & 57.5 & 72.1 & 63.9 & 57.4 & 37.8 & 70.2 & 61.5 & 83.3 & 71.1 & 64.8\\
  $\text{MBQ}$ & & 41.2 & 66.9 & \textbf{65.0} & 76.7 & \textbf{73.4} & 64.6 & 43.3 & \textbf{74.1} & 64.3 & \textbf{86.0} & 74.8 & 68.5\\
  \rowcolor[HTML]{EAE8FD}$\text{MASQuant(ours)}$ &  &  \textbf{46.7} & \textbf{67.2} & 62.7 & \textbf{77.9} & 69.2 & \textbf{64.7} & \textbf{43.3} & 72.8 & \textbf{66.4} & 85.7 & \textbf{77.0} & \textbf{69.0} \\
\hline
   $\text{SQ}$ & \multirow{3}{*}{W4A6} & 22.5 & 66.7 & 52.9 & 69.9 & 55.9 & 53.6 & 28.9 & 67.7 & 60.5 & 78.6 & 69.3 & 61.0 \\
   $\text{MBQ}$ &  & 38.7 & 65.6 & \textbf{60.5} & 64.7 & \textbf{69.1} & 59.7 & \textbf{30.0} & \textbf{71.7} & 59.8 & \textbf{80.1} & 72.9 & 62.9\\
  \rowcolor[HTML]{EAE8FD}$\text{MASQuant(ours)}$ & & \textbf{40.0} & \textbf{66.7} & 56.2& \textbf{71.2} & 65.3 & \textbf{59.9} & 29.7 & 70.3 & \textbf{62.6} & 79.7 & \textbf{72.9} & \textbf{63.0} \\
\bottomrule
\noalign{\vspace{0.1em}}\noalign{\vspace{0.1em}}
\end{tabular}
}
\vspace{-1.0em}
\end{table*}

\begin{table*}[t!]
\setlength{\tabcolsep}{4pt}
\centering
\caption{\small Comparison of MASQuant with existing quantization methods on omni-modal MLLMs (vision, audio and text). SQ: SmoothQuant. Libri: Librispeech. Wen: Wenetspeech. The best results are highlighted in \textbf{bold}.}
\label{tab:main_omni}
\resizebox{\textwidth}{!}{%
\begin{tabular}{@{}l|c!{\vrule width 1pt}|cc|c|c!{\vrule width 1pt}|cc|c|c@{}}
\noalign{\vspace{0.1em}}\toprule
  &  &\multicolumn{4}{c}{\textbf{Qwen2.5-Omni-3B}} & \multicolumn{4}{c}{\textbf{Qwen2.5-Omni-7B}}  \\
\hline
 \textbf{Method} & Bits  &\multicolumn{2}{c}{Audio-Text} & {Vision-Text} & Vision-Audio-Text  & \multicolumn{2}{c}{Audio-Text} & {Vision-Text} & Vision-Audio-Text\\
   & &Libri & Wen & MMMU & Omnibench & Libri & Wen  & MMMU & Omnibench  \\
  & & WER$\downarrow$ & WER$\downarrow$ & Acc$\uparrow$  & Acc$\uparrow$ & WER$\downarrow$ & WER$\downarrow$ & Acc$\uparrow$ & Acc$\uparrow$ \\
\noalign{\vspace{0.1em}}\hline\noalign{\vspace{0.1em}}
   $\text{Dense}$ & FP16  & 3.9 & 7.5 & 43.3 & 43.8 & 2.9 & 7.1  & 50.0 & 45.3   \\
\hline
\noalign{\vspace{0.1em}}
 $\text{RTN}$ & \multirow{4}{*}{W4A16}& 4.4 & 9.2  & 32.2 & 39.6 & 3.4 & 6.7 & 44.4 & 43.8 \\
  $\text{AWQ}$ & & 9.4 & 8.0 & 36.7 & 43.8 & 8.7 & 7.1 & 47.8 & 45.3    \\
  $\text{MBQ}$  & & 8.2 & \textbf{7.0} & 38.9 & 42.6 & 3.5 & 6.6 & 47.8 & 44.4 \\
  \rowcolor[HTML]{EAE8FD}$\text{MASQuant(ours)}$ &  & \textbf{4.0} & 8.0  & \textbf{39.6} & \textbf{46.9} &  \textbf{2.7} & \textbf{7.4} & \textbf{48.8} & \textbf{45.3} \\
\hline
  $\text{RTN}$ & \multirow{4}{*}{W8A8} & 3.8 & 7.6  &  41.3 & 37.7 & 8.2 & 7.0 & 48.9 & 38.4\\
  $\text{SQ}$ & & \textbf{3.6} & \textbf{7.5}  & 41.7 & 44.5 & 8.2 & 7.3 & 48.8 & 47.6 \\
  $\text{MBQ}$ & & 8.9 & 7.8  & 41.8 & 46.1 & 8.1 & 7.1 & \textbf{50.0} & 46.4\\
  \rowcolor[HTML]{EAE8FD}$\text{MASQuant(ours)}$ &  & 3.8 & 7.9 & \textbf{42.3} & \textbf{46.1} & \textbf{2.8} & \textbf{6.9} & 49.9  & \textbf{46.9} \\
\hline
 $\text{RTN}$ & \multirow{4}{*}{W4A8} & 109.7 & 105.6 & 28.9 & 29.7 & 9.0 & 8.7  & 42.2 & 35.2 \\
   $\text{SQ}$ & & 77.4 & 94.2 & 30.0 & 27.3 & 8.6 & 8.3 & 42.2 & 36.7 \\
  $\text{MBQ}$ & & 9.5 & \textbf{8.5} & 27.8 & 36.7 & 3.8 & 8.2  & 47.8 & 40.6 \\
  \rowcolor[HTML]{EAE8FD}$\text{MASQuant(ours)}$ & & \textbf{3.6} & 8.7  & \textbf{36.7} & \textbf{41.4} & \textbf{2.9} & \textbf{8.0} & \textbf{48.8} & \textbf{43.8}\\
\hline
 $\text{RTN}$ & \multirow{4}{*}{W4A6} & 87.8 & 99.6 & 28.9 & 23.4 & 11.8 &  15.7 & 28.9 & 40.6 \\
   $\text{SQ}$ &  & 87.8 & 99.5  & 28.9 & 18.8 & 5.8 & 19.8  & \textbf{37.8} & 35.9 \\
   $\text{MBQ}$ & & 10.8 & 10.4 & 31.4 & \textbf{46.9} & 6.4 & 14.8 & 33.3 & 42.1\\
  \rowcolor[HTML]{EAE8FD}$\text{MASQuant(ours)}$ & & \textbf{3.7} & \textbf{8.9}  & \textbf{32.2} & 42.2 & \textbf{4.7} & \textbf{8.7}  & 36.8 & \textbf{42.2}\\
\bottomrule
\noalign{\vspace{0.1em}}\noalign{\vspace{0.1em}}
\end{tabular}
}
\vspace{-1.0em}
\end{table*}
\section{Experiments}
\subsection{Experimental Setups}

We evaluate on {Qwen2.5-VL}~\citep{qwen2.5-vl} and Qwen2.5-Omni~\citep{qwen2.5omni}, multimodal LLMs supporting text, audio, and vision. The architecture of Omni MLLMs includes 675M Vision Transformer-based~\citep{vit} vision encoder, Whisperlarge-v3~\citep{wisper} based audio encoder, LLM Transformer decoder Thinker, speech output Talker, and Streaming Codec Decoder. For generality, we quantize only the LLM component {Thinker}. We evaluate several state-of-the-art channel-wise smoothing-based quantization methods: {AWQ}~\citep{awq}, {SmoothQuant}~\citep{smoothquant}, and {MBQ}~\citep{mbq}. 
We evaluate model performance across multiple multimodal benchmarks: {Audio-to-Text Tasks}: Evaluated on {Librispeech~\citep{libri} dataset test-other split} and {Wenetspeech~\citep{wenetspeech} dataset test-net split}, using {Word Error Rate (WER)} as metric. {Visual Reasoning Tasks}: Evaluated on {OCRBench~\citep{ocrbench}}, TextVQA~\citep{textvaq}, Vizwiz~\citep{vizwiz}, ScienceQA~\citep{scienceqa} and {MMMU~\citep{mmmu}}.{Multimodal Reasoning Tasks}: Evaluated on {OmniBench~\citep{omnibench}}, covering joint text-audio-visual reasoning.

\subsection{Vision-Language MLLMs}
Table~\ref{tab:main_vl} evaluates MASQuant on Qwen2.5-VL models. MASQuant matches FP16 performance on both model sizes at W8A8, suggesting that MLLMs can be quantized to 8 bits without accuracy loss when modality-specific characteristics are properly handled. RTN completely fails on W4A8, while SmoothQuant severely degrades. This failure pattern reveals that modality dominance becomes catastrophic at aggressive quantization levels, where smoothing factors dictated by the dominant modality harm the weaker modality. 

\subsection{Vision-Audio-Language MLLMs}
Table~\ref{tab:main_omni} evaluates quantization on omni-modal models handling vision, audio, and text. The results reveal that modality dominance intensifies as modality diversity increases. At W4A8, SmoothQuant's audio performance collapses catastrophically on 3B: Librispeech WER jumps from 3.9 to 77.4, and Wenetspeech from 7.5 to 94.2. This 20× degradation, far worse than vision-language failures, demonstrates that audio is completely suppressed when competing with vision and text for smoothing resources. MASQuant maintains near-FP16 audio quality, confirming that per-modality smoothing prevents this collapse. Audio's vulnerability is intuitive: its smaller activation magnitudes make it the first victim of vision-dominated smoothing factors.

\subsection{Analysis}
\paragraph{Modality Dominance.}
We investigate modality dominance by applying SmoothQuant to Qwen2.5-VL and Qwen2.5-Omni. Figure~\ref{fig:modality_dominance} shows that visual tokens exhibit significantly larger activations than text in both attention and MLP layers. Consequently, smoothing factors track visual distributions while mismatching text by orders of magnitude. This confirms modality dominance is pervasive in MLLMs. We further analyze $\alpha_i$ distribution, which determines MAS's SQNR gain. Figures~\ref{fig:range_ratio_omni} and~\ref{fig:range_ratio_vl} show non-uniform $\alpha_i$ across components in both models, validating consistent improvements from MAS.

\begin{figure}[h!]
    \centering
\begin{subfigure}{0.49\columnwidth}
\centering
\includegraphics[width=\textwidth]{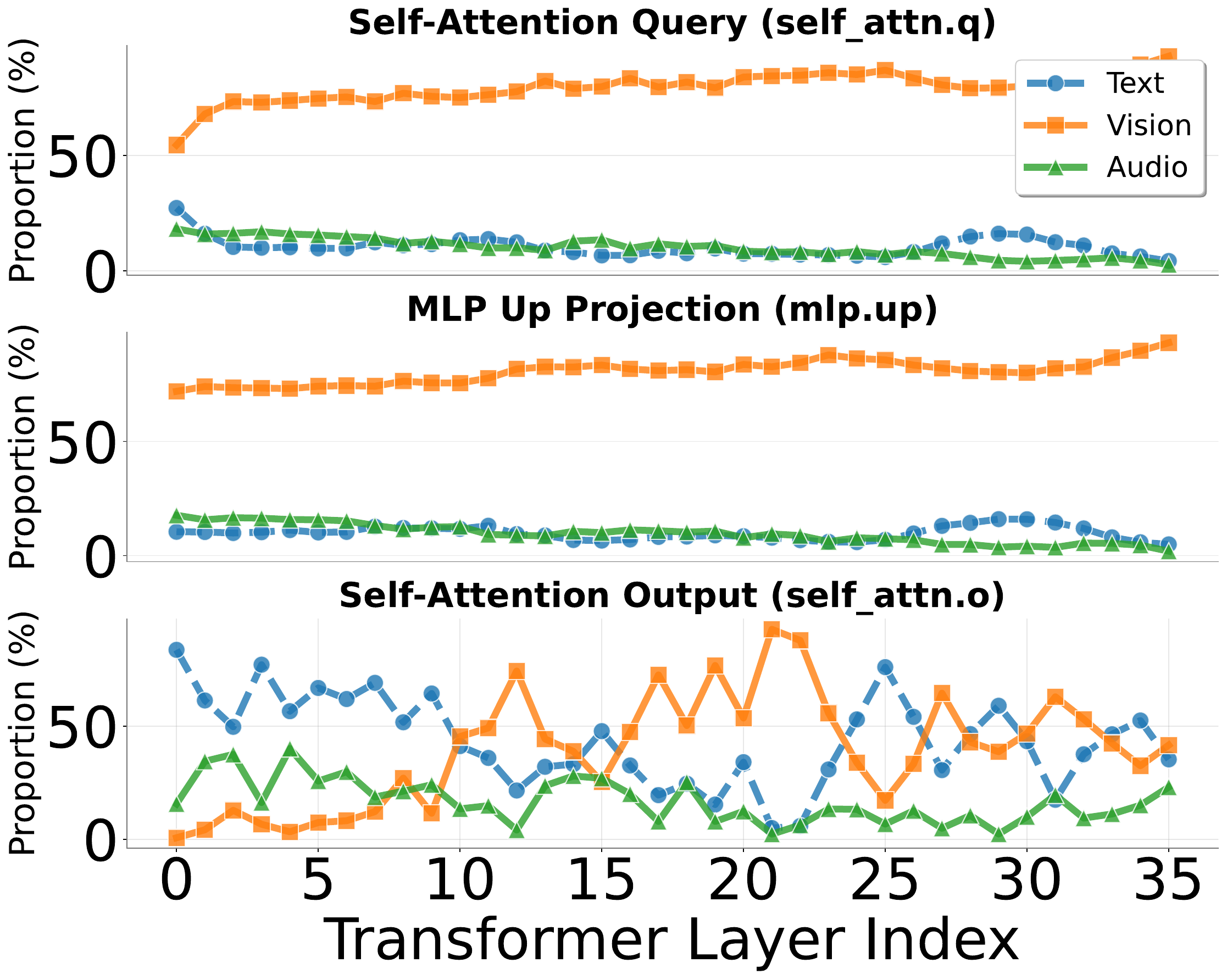}
\vspace{-3mm}
  \caption{\small Qwen2.5-Omni-3B}
  \label{fig:modality_dominance_omni}
\end{subfigure}
\begin{subfigure}{0.49\columnwidth}
\centering
\includegraphics[width=\textwidth]{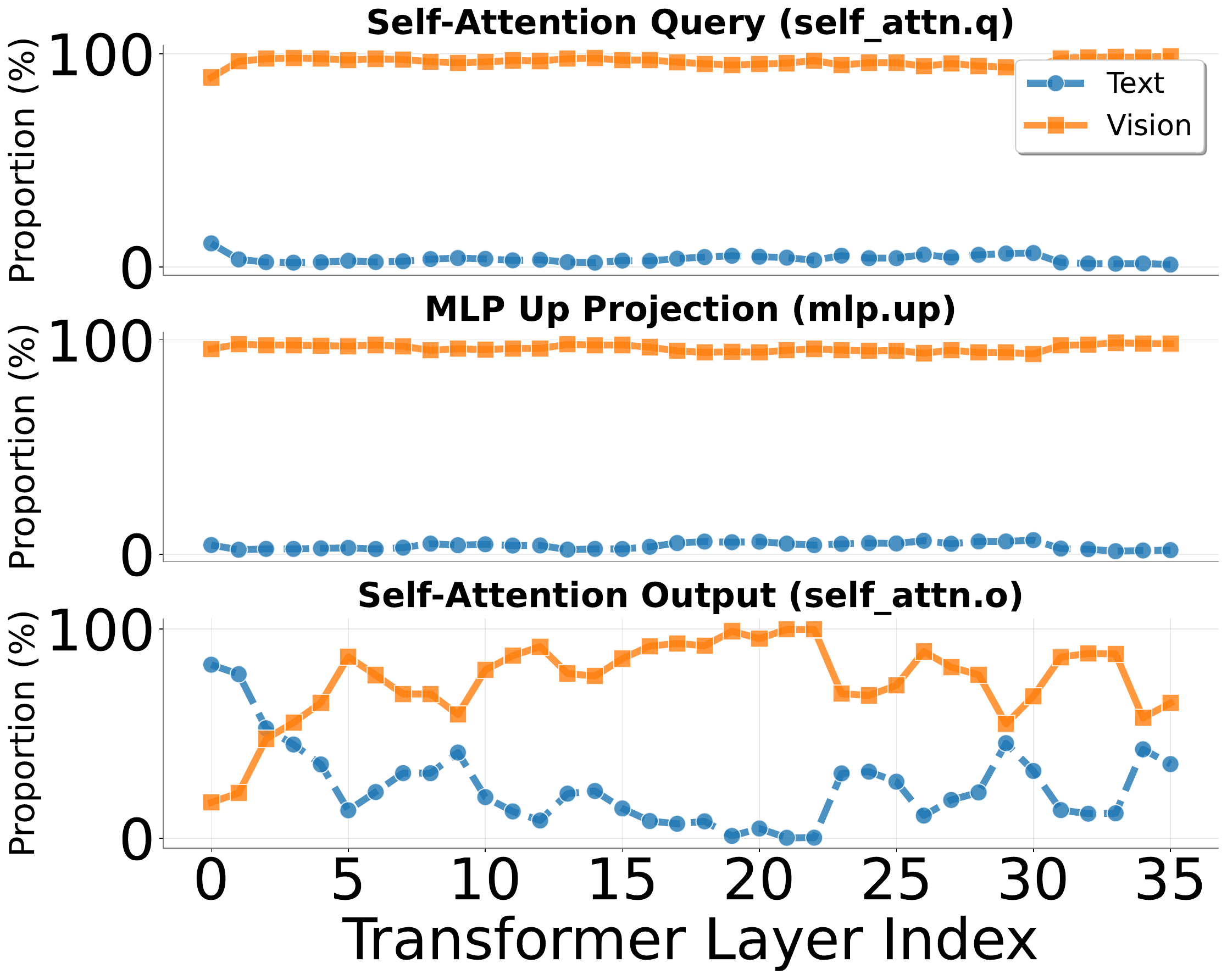}
\vspace{-3mm}
  \caption{\small Qwen2.5-VL-3B}
  \label{fig:modality_dominance_vl}
\end{subfigure}
\begin{subfigure}{0.49\columnwidth}
\centering
\includegraphics[width=\textwidth]{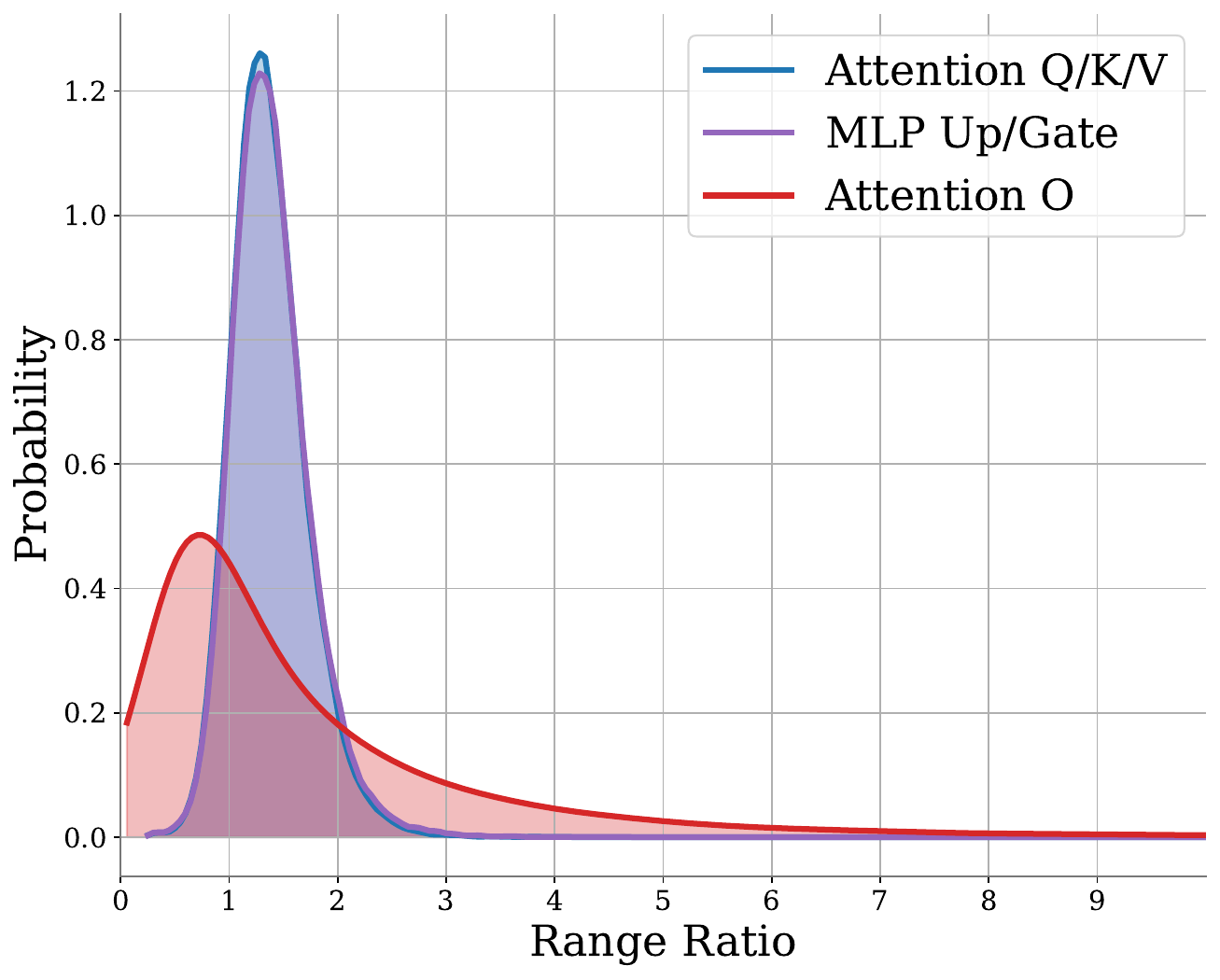}
\vspace{-3mm}
  \caption{\small Qwen2.5-Omni-3B}
  \label{fig:range_ratio_omni}
\end{subfigure}
\begin{subfigure}{0.49\columnwidth}
\centering
\includegraphics[width=\textwidth]{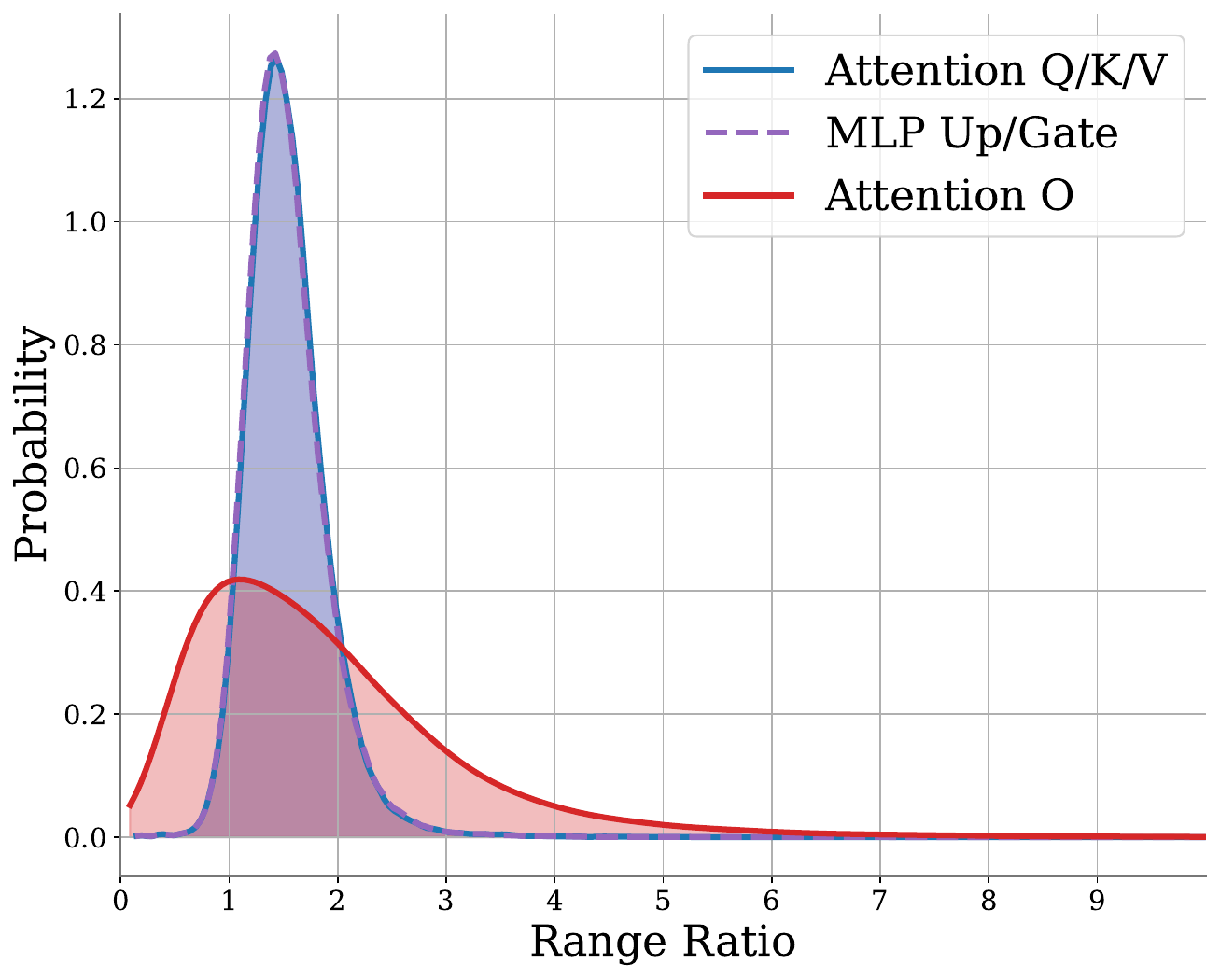}
\vspace{-3mm}
  \caption{\small Qwen2.5-VL-3B}
  \label{fig:range_ratio_vl}
\end{subfigure}
\caption{Percentage of unified smoothing factors from different modalities using SmoothQuant on Omni and VL MLLMs.}
\label{fig:modality_dominance}
\end{figure}

\paragraph{Effective Rank.}
CMC relies on weight differences having low-rank structure after whitening. We verify this on Qwen2.5-VL/Omni-3B using effective rank (lower is better). As shown in Figure~\ref{fig:effective_rank}, whitening reduces effective rank substantially, confirming our design.

\begin{figure}[h!]
    \centering
\begin{subfigure}{0.49\columnwidth}
\centering
\includegraphics[width=\textwidth]{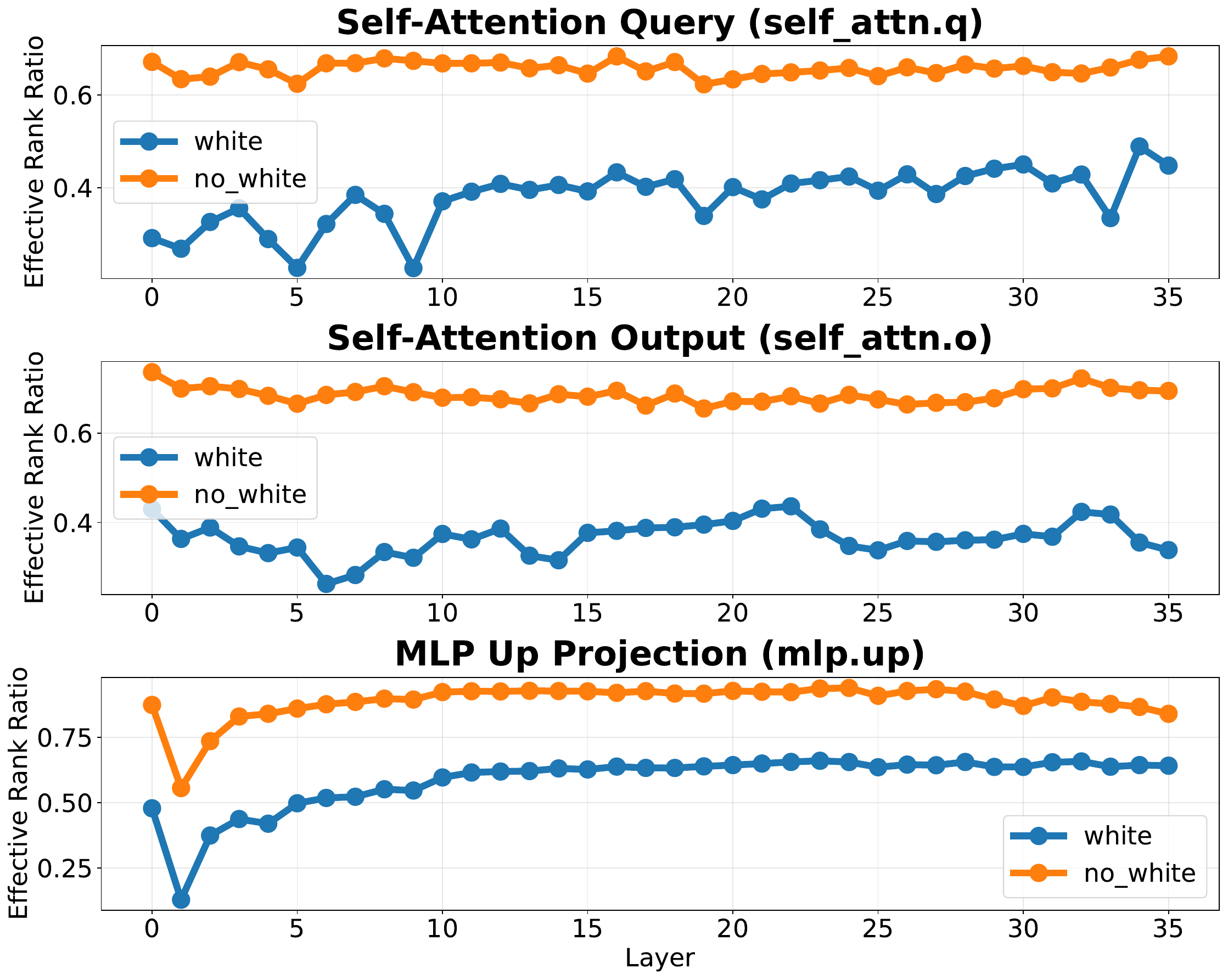}
\vspace{-3mm}
  \caption{\small Qwen2.5-VL-3B}
  \label{fig:effective_rank_vl}
\end{subfigure}
\begin{subfigure}{0.49\columnwidth}
\centering
\includegraphics[width=\textwidth]{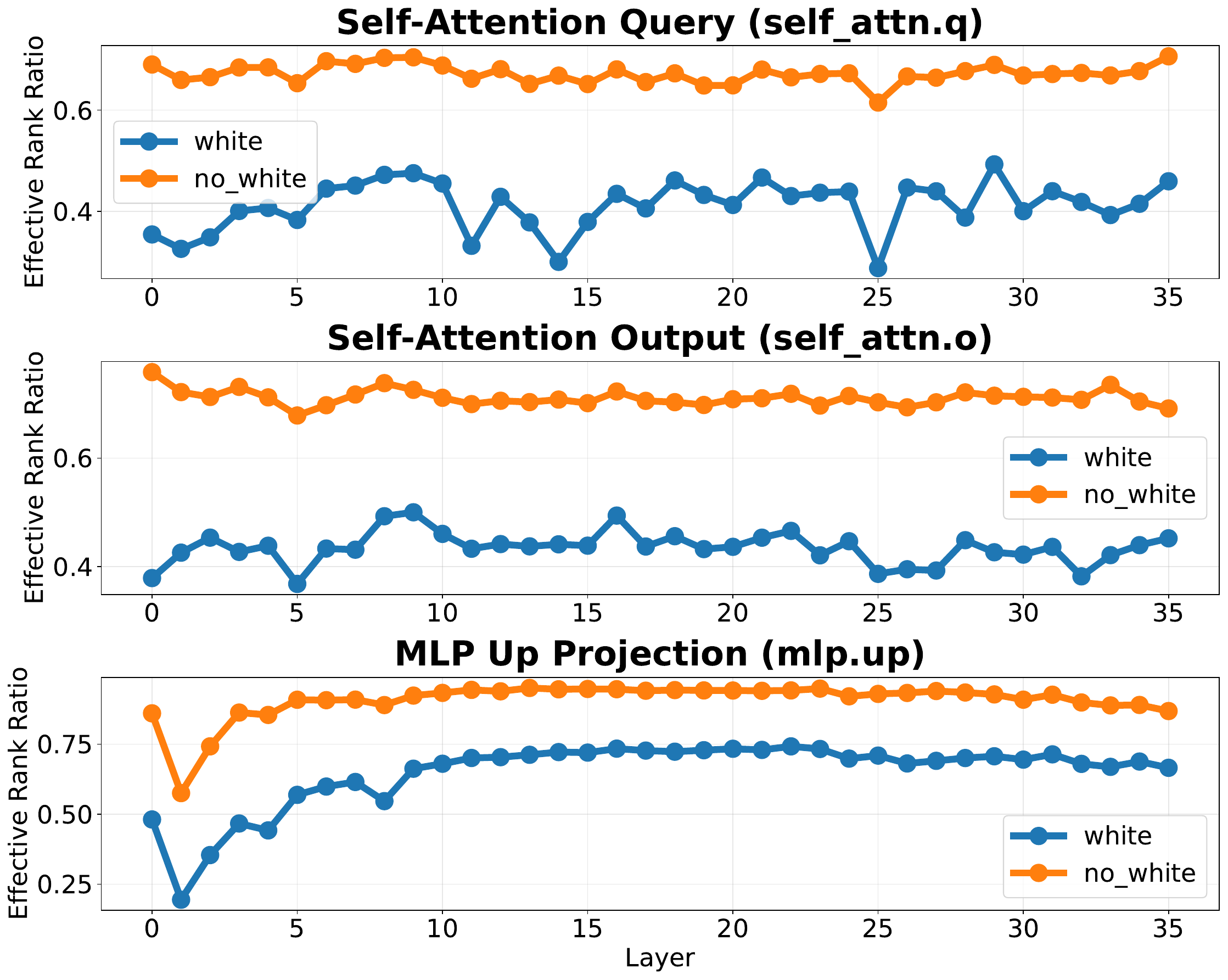}
\vspace{-3mm}
  \caption{\small Qwen2.5-Omni-3B}
  \label{fig:effective_rank_omni}
\end{subfigure}
\caption{Effective ranks of $\Delta\mathbf{W}$ is reduced across layers after SVD-based Whitening b SQNR improves as the rank ratio increases on both Qwen2.5-VL-3B and Qwen2.5-Omni-3B.}
\label{fig:effective_rank}
\end{figure}

\paragraph{The Effect of MAS.}
Table~\ref{tab:Effects} reveals the necessity of modality-aware smoothing in W4A8 quantization. The most striking result is on LibriSpeech: uniform smoothing yields 77.4 WER while MAS achieves 3.8 WER. This dramatic gap exposes a fundamental failure mode where uniform smoothing catastrophically degrades the weaker modality when forced to compromise between modalities with disparate activation ranges. Notably, learnable optimization amplifies rather than closes this gap.
\begin{table}[h!]
\centering
\caption{Effects of Modality-Aware Smoothing (W4A8).}
\vspace{-3mm}
\label{tab:Effects}
\begin{tabular}{@{}lc|cc|cc@{}}
\toprule
& & \multicolumn{2}{c|}{Omni-3B} & \multicolumn{2}{c}{VL-3B} \\
\cmidrule(lr){3-4} \cmidrule(lr){5-6}
Method & Opt. & Libri & OmniB. & T-VQA. & MMMU \\
\midrule
Uniform & & 77.4 & 27.3 & 51.3 & 25.6 \\
MAS & & \textbf{3.8} & \textbf{36.7} & \textbf{54.7} & \textbf{28.9} \\
\midrule
Uniform & \checkmark & 6.0 & 33.3 & 65.0 & 33.3 \\
MAS & \checkmark & \textbf{3.6} & \textbf{47.7} & \textbf{68.2} & \textbf{46.7} \\
\bottomrule
\end{tabular}
\end{table}
\vspace{-1pt}
\paragraph{Modality Loss Weight.}
Table~\ref{tab:mas_comparison_loss} shows ablation results of $\lambda_t$ and $\lambda_v$ on Qwen2.5-VL-3B. Equal smoothing works best: 17.2 PPL and 56.9\% average accuracy. Reducing vision smoothing to $\lambda_v{=}0.5$ drops MMMU from 37.8 to 30.0 with minimal text benefit. Further reducing to 0.1 collapses PPL to 33.7. Reducing text smoothing ($\lambda_t{=}0.5$) degrades all metrics. In this work, we set $\lambda_t{=}\lambda_v{=}1.0$. 

\begin{table}[t]
\centering
\caption{Effects of Modality Loss Weight~(W4A8).}
\vspace{-3mm}
\label{tab:mas_comparison_loss}
\begin{tabular}{@{}cc|c|cccc@{}}
\toprule
&  &  \multicolumn{5}{c}{Qwen2.5-VL-3B} \\
$\lambda_{t}$ & $\lambda_{v}$ & PPL & MMMU & T-VQA. & S-QA & Avg\\
\midrule
1.0 & 1.0 & 17.2 & 37.8 & 55.2 & 77.8 & 56.9\\
\hline
1.0 & 0.5 & 17.8 & 30.0 & 57.8 & 80.0 & 55.9 \\
\hline
1.0 & 0.1 & 33.7 & 32.2 & 65.4 & 76.7 & 58.1 \\
\hline
0.5 & 1.0 & 17.3 & 33.3 & 53.9 & 68.9 & 52.0 \\
\bottomrule
\end{tabular}
\end{table}
\vspace{-3mm}
\paragraph{Training Epochs.} Table~\ref{tab:mas_comparison} shows convergence on Qwen2.5-VL-3B. Performance improves rapidly: PPL drops from 23.9 to 17.0. Average accuracy peaks at epoch 2, then slightly declines. We use 2 epochs—a good trade-off between efficiency and performance.

\begin{table}[h]
\centering
\caption{Effects of Modality-Aware Smoothing (W4A8).}
\vspace{-3mm}
\label{tab:mas_comparison}
\begin{tabular}{@{}c|c|cccc@{}}
\toprule
&  \multicolumn{4}{c}{Qwen2.5-VL-3B} \\
Epoch & PPL & MMLU & T-VQA. & S-QA & Avg\\
\midrule
1 & 23.9 & 38.9 & 52.2 & 79.5 & 56.9 \\
\hline
2 & 18.6 & 37.8 & 67.5 & 78.4 & 61.2 \\
\hline
5 & 17.6 & 34.4 & 65.9 & 77.4 & 59.2 \\
\hline
10 & 17.0 & 37.8 & 68.2 & 76.5 & 60.8 \\
\bottomrule
\end{tabular}
\end{table}
\vspace{-5mm}
\paragraph{The Effect of CMC.}
Figure~\ref{fig:cmc_sqnr_effec_rank} plots SQNR vs. rank ratio. CMC dominates the Non-Whitened baseline in the low-rank regime, reducing the rank required to match MBQ by 4x. Specifically, CMC surpasses MAS at ratio of 0.08 and saturates near 3.5 SQNR, whereas the baseline struggles to match MBQ until a ratio of 0.4.

\begin{table}[h!]
\centering
  \caption{Computation cost and memory at decoding phase, where $d$ is hidden size, $r$ is rank, and $m$ is the number of extra modalities.}
  \vspace{-3mm}
 \begin{tabular}{@{}c|cc@{}}
    \toprule
    Base modality  & Computation & Memory \\
    \hline
    Text  & $d$ & 0\\
    \hline
    Others  & $d$ + $2rd$ & $2mrd$ \\
    \bottomrule
    \end{tabular}
\label{tb:computation cost}
\end{table}
\vspace{-8mm}
\paragraph{Base Modality Choice.} We choose text as the base modality to decouple CMC from decoding. As shown in Table~\ref{tb:computation cost}, other modalities would require CMC involvement, leading to extra costs of computation and memory.

\begin{figure}[h!]
    \centering
\begin{subfigure}{0.9\columnwidth}
\centering
\includegraphics[width=\textwidth]{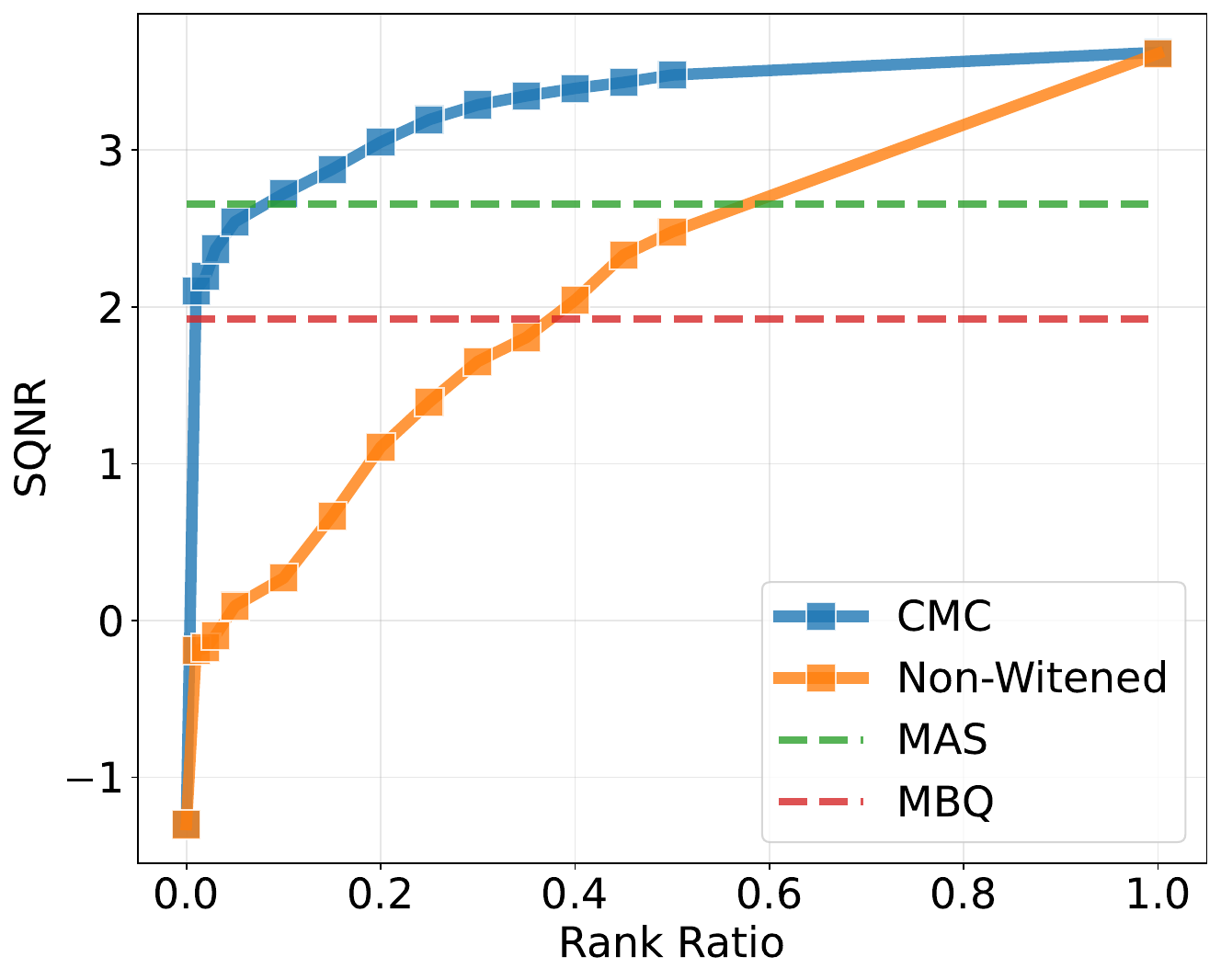}
\vspace{-3mm}
\end{subfigure}
\vspace{-3mm}
\caption{SQNR as a function of rank ratio for CMC on Qwen2.5-Omni-3B under W4A6 quantization. MAS applies independent modality-specific smoothing. MBQ employs a uniform factor optimized by modality balance reconstruction.}
\label{fig:cmc_sqnr_effec_rank}
\end{figure}

\subsection{Inference Speed Up}
\vspace{-2mm}
To validate practical efficiency, we implement a custom CUDA kernel based on Nunchaku~\citep{li2024svdquant} that fuses projection and quantization operators to minimize memory access. A multi-modal mask efficiently manages conditional low-rank execution. Table~\ref{tab:performance} shows that MASQuant achieves a 2.5x speedup over FP16 with marginal latency overhead (5--10\%) compared to MBQ, while maintaining identical decoding latency.

\vspace{-3mm}
\begin{table}[h]
\centering
\caption{End-to-end prefill-stage performance of Qwen2.5-VL-7B on Desktop RTX 4090 with fused GPU kernels (sequence length = 2048) under W4A4 setting. MAS: MASQuant. BS: Batch Size.}
\vspace{-10pt}
\label{tab:performance}
\resizebox{\linewidth}{!}{
\begin{tabular}{@{}c|c|c|cc|cc@{}}
\toprule
BS & Method & Rank  & Prefill  & Speedup & Mem & Mem  \\
 &  &  Ratio &  (ms) &  & (GB) &  Saving \\
\midrule
 & FP16 & - & 191.82 & / & 13.73 & /\\
\hline
\multirow{4}{*}{1} & MBQ & - &68.65 & \textbf{2.79}$\times$ & 4.85 & \textbf{2.83}$\times$ \\
 & MAS & 0.01 & 71.62 & \textbf{2.67}$\times$ & 4.97 & \textbf{2.76}$\times$ \\
 & MAS & 0.02 & 72.89 & \textbf{2.63}$\times$ & 5.04 & \textbf{2.73}$\times$ \\
 & MAS & 0.05 & 77.10 & \textbf{2.49}$\times$ & 5.37 & \textbf{2.56}$\times$ \\
\hline
 & FP16 & - & 2146.01 & / & 17.27 & /\\
\hline
 \multirow{4}{*}{8} & MBQ & - & 643.68 & \textbf{3.33}$\times$ & 8.89 & \textbf{1.94}$\times$ \\
 & MAS & 0.01 & 649.29 & \textbf{3.30}$\times$ & 9.02 & \textbf{1.92}$\times$ \\
 & MAS & 0.02 & 657.68 & \textbf{3.26}$\times$ & 9.08 & \textbf{1.90}$\times$ \\
 & MAS & 0.05 & 696.44 & \textbf{3.07}$\times$ & 9.42 & \textbf{1.83}$\times$ \\
\bottomrule
\end{tabular}
}
\end{table}
\vspace{-5mm}
\section{Conclusion}
In this work, we identify smoothing misalignment as the primary obstacle preventing channel-wise smoothing from applying to MLLMs. To address this, we propose the MASQuant which applies Modality-aware Smoothing and Cross-Modal Compensation to support modality-specific smoothing factors and single quantized weight during inference. Our approach is simple and highly effective. We demonstrate its efficacy on multi-modal benchmark.
{
    \small
    \bibliographystyle{ieeenat_fullname}
    \bibliography{main}
}


\end{document}